\documentclass[10pt,journal,compsoc]{IEEEtran}

\usepackage{amsthm}

\newtheorem{thm}{Theorem}[]

\newtheorem{fact}[thm]{Fact}

\usepackage{booktabs}
\usepackage{multirow}

\usepackage{ifpdf}

\ifCLASSOPTIONcompsoc
 \usepackage[nocompress]{cite}
\else
 \usepackage{cite}
\fi
\ifCLASSINFOpdf
  \usepackage[pdftex]{graphicx}
\else
  \usepackage[dvips]{graphicx}
\fi
\usepackage{amsmath}
\usepackage{amssymb}
\usepackage{algorithmic}

\usepackage{array}

\ifCLASSOPTIONcompsoc
 \usepackage[caption=false,font=footnotesize,labelfont=sf,textfont=sf]{subfig}
\else
 \usepackage[caption=false,font=footnotesize]{subfig}
\fi
\usepackage{fixltx2e}
\usepackage{dblfloatfix}

\ifCLASSOPTIONcaptionsoff
 \usepackage[nomarkers]{endfloat}
\let\MYoriglatexcaption\caption
\renewcommand{\caption}[2][\relax]{\MYoriglatexcaption[#2]{#2}}
\fi
\usepackage{url}

\begin{document}
%
\title{AliExpress Learning-To-Rank: Maximizing Online Model Performance without Going Online}
%
%
%
%

\author{Guangda Huzhang, Zhen-Jia Pang, Yongqing Gao, Yawen Liu, Weijie Shen, Wen-Ji Zhou, Qing Da, An-Xiang Zeng, Han Yu, Yang Yu~\IEEEmembership{Member,~IEEE}, Zhi-Hua Zhou~\IEEEmembership{Fellow,~IEEE}
\IEEEcompsocitemizethanks{
\IEEEcompsocthanksitem Guangda Huzhang, Wen-Ji Zhou, Qing Da, and An-Xiang Zeng are with Alibaba.
\IEEEcompsocthanksitem Zhen-Jia Pang, Yongqing Gao, Yawen Liu, Weijie Shen, Yang Yu, and Zhi-Hua Zhou are with State Key Laboratory for Novel Software Technology, Nanjing University, Nanjing 210023, China.
\IEEEcompsocthanksitem Han Yu is with Nanyang Technological University, Singapore.
\IEEEcompsocthanksitem The first two authors contributed equally to this work. Corresponding authors: Wenji Zhou (Alibaba, eric.zwj@alibaba-inc.com) and Yang Yu (Nanjing University, yuy@nju.edu.cn).}%
\thanks{Manuscript revised XXX XXXX, 2020.}}

%
%

\markboth{Journal of \LaTeX\ Class Files,~Vol.~14, No.~8, August~2015}%
{Shell \MakeLowercase{\textit{et al.}}: Bare Demo of IEEEtran.cls for Computer Society Journals}
%



\IEEEtitleabstractindextext{%
\begin{abstract}

Learning-to-rank (LTR) has become a key technology in E-commerce applications. Most existing LTR approaches follow a supervised learning paradigm from offline labeled data collected from the online system. 
However, it has been noticed that previous 
LTR models can have a good validation performance over offline validation data but have a poor online performance, and vice versa, which implies a possible large inconsistency between the offline and online evaluation. We investigate and confirm in this paper that such inconsistency exists and can have a significant impact on AliExpress Search. Reasons for the inconsistency include the ignorance of item context during the learning, and the offline data set is insufficient for learning the context. Therefore, this paper proposes an evaluator-generator framework for LTR with item context.
The framework consists of an evaluator that generalizes to evaluate recommendations involving the context, and a generator that maximizes the evaluator score by reinforcement learning, and a discriminator that ensures the generalization of the evaluator. 
Extensive experiments in simulation environments and AliExpress Search online system show that, firstly, the classic data-based metrics on the offline dataset can show significant inconsistency with online performance, and can even be misleading. Secondly, the proposed evaluator score is significantly more consistent with the online performance than common ranking metrics. 
Finally, as the consequence, our method achieves a significant improvement (\textgreater$2\%$) in terms of Conversion Rate (CR) over the industrial-level fine-tuned model in online A/B tests.
\end{abstract}

\begin{IEEEkeywords}
Learning-To-Rank, evaluation, offline-online inconsistency
\end{IEEEkeywords}}

\maketitle

\IEEEdisplaynontitleabstractindextext

%
\IEEEpeerreviewmaketitle

\IEEEraisesectionheading{\section{Introduction}\label{sec:introduction}}

%
%
%
%
\IEEEPARstart
{L}{earning-to-rank (LTR)} has been the focus of online search engine and recommender system research for enhancing profitability. In a given scenario, existing LTR approaches generally assume that the click-through rate (or conversion rate) of an item is an intrinsic property of the item, which needs to be accurately discovered with sophisticated predictive models. 
Under such an assumption, it is reasonable to focus on data-based ranking metrics such as Area Under Curve (AUC) and Normalized Discounted Cumulative Gain (NDCG) to evaluate model performance, which has led to many LTR models closely match the labels in the offline data.

\begin{figure*}[t]
\centering
\includegraphics[scale=0.345]{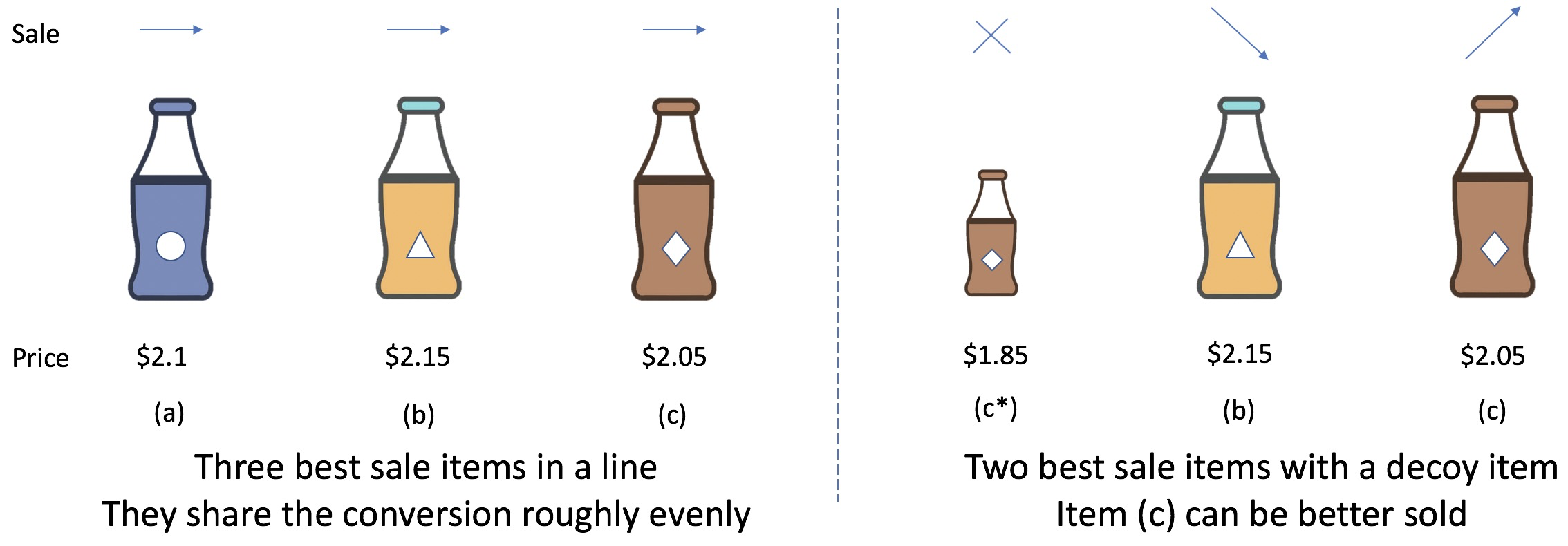}
\caption{An example of \emph{the decoy effect}, or \emph{the anchoring effect}. Standard LTR solutions produce the order on the left as the items are well sold, but the creative orders on the right may achieve better overall performance. Item (c*) plays as a decoy (a smaller but expensive version of item (c) ) that will not be purchased, while improves the impression of item (c).
}
\label{fig:decoy}
\end{figure*}

However, practitioners have reported the ``offline-online inconsistency problem'' in recent works~\cite{rohde2018recogym,beel2013comparative,rossetti2016contrasting,mcnee2006being}: \textit{a new model achieving significant improvement on offline metrics may not achieve the same improvement when deployed online}. 
One major reason of the inconsistency is that purchase intention of customers is influenced by the context.
Here is an example: if an item is shown together with similar but more expensive items, the likelihood of purchase increases, which is known as the \emph{decoy effect}. Figure \ref{fig:decoy} illustrates this with an example that context may drastically change behaviors of customers. Another commonly observed phenomenon is that, once a customer clicks an item, the chance of clicking the next item usually decreases ~\cite{wang19baidu}. 
Thus, the inconsistency is introduced since the labels logged from feedback of customers may no longer be true when the order of items changes, i.e. if we present a new order of the same items to the same customer, the same customer may make a different purchase decision. Therefore, we need to carefully consider the context in E-commerce rankings.

To address the issue of the contextual influence, the \emph{re-ranking} strategy has been proposed~\cite{Zhuang18} for practical LTR applications. The ranking system first selects a small set of candidate items; then, it determines the order of these candidate items using a re-ranking model. Different from classic LTR models, re-ranking models can enhance the understanding of the candidate items in a holistic manner. 
Most of re-ranking models are trained by supervised learning~\cite{pei2019personalized,seq2slate,Zhuang18}, and are evaluated with data-based metrics on offline datasets. However, in context-sensitive scenarios, classic re-ranking models, which follow the supervised learning paradigm, still do not explore the combinatorial space to find the best permutation of items. Thus, these models are less possible to optimize actual performance metrics, such as Conversion Rate (CR) and Gross Merchandise Volume (GMV), after deployment. For example, such models cannot generate creative order (as shown on the right of Figure \ref{fig:decoy}) which can potentially have a better performance.

In order to address the aforementioned problem, we require an evaluation approach beyond the static dataset as well as an exploration approach beyond the supervised learning paradigm. The ideal approach is to learn through interacting with the real customers in the online environment. However, such a direct interaction approach can be highly risky and costly as repeated (possibly unsuccessful) trial-and-error may negatively impact customer experience and revenue. 

In this paper, we present the Evaluator-Generator Re-ranking \emph{(EG-Rerank)} framework for E-commerce LTR systems to address the offline-online inconsistency problem while avoiding the pitfalls of direct online interaction-based evaluation. EG-Rerank consists of \emph{an evaluator} and \emph{a generator}.
The evaluator aims at modeling user feedback (i.e. predicting the probabilities of purchase given a list). It works as a virtual environment of the real online environment, and its scores can better reflect intentions of customers than labels used in static datasets.
The generator is then learned according to the evaluator through reinforcement learning.

In EG-Rerank, we have noticed that the evaluator is learned from the offline labeled data and may not generalize well for generated ranks that are very different from the data. Therefore, we introduce \emph{a discriminator} model to provide a self-confidence score for the evaluation, which is trained adversarially to the generator and guides the generator to produce lists not far from the data. Therefore, the discriminator restricts the exploration space of the generator and ensures a more trustable update. We name EG-Rerank with a discriminator as \textit{EG-Rerank+}.


This work makes the following contributions to E-commerce LTR:
\begin{itemize}
\setlength{\itemsep}{0pt}
	\item Through experiments in a simulation environment and
	AliExpress Search online system, which is one of the world's largest international retail platforms and has more than $150$ million buyers from more than $220$ countries\footnote{The information about our platform can be found in \url{ https://sell.aliexpress.com/__pc/4DYTFsSkV0.htm}}, we demonstrate the significance of the offline-online inconsistency problem. 
	\item We further show that the EG-Rerank evaluator can be a more robust objective compared to existing metrics on offline data, and can serve as a substitution of these metrics.
	\item We present the EG-Rerank and EG-Rerank+ as two approaches of the evaluator-generator framework for ranking applications. 
	\item In AliExpress Search, EG-Rerank+ consistently improves the conversion rate by 2\% over the fine-tuned industrial-level re-ranking model in online A/B tests, which translates into a significant improvement of revenue increasing.
\end{itemize}

\section{Related Work}
Learning-to-rank (LTR) solves item ranking problems through machine learning. There are several types of LTR models, including point-wise, pair-wise, list-wise, and so on. Point-wise models ~\cite{cossock08point,li07point} treat the ranking problems as classification or regression tasks for each item. Pair-wise models~\cite{joachims02pair,Burges:ranknet,burges2010ranknet} convert the original problem into the internal ranking of pairs. List-wise models ~\cite{cao07list,xia08list,ai2018learning} use well-designed loss functions to directly optimize the ranking criteria. Group-wise models~\cite{ai2019learning} and page-wise models~\cite{zhao2018deep} are proposed recently, which are similar to re-ranking models~\cite{Zhuang18}. 
However, all the above-cited methods focus on optimizing offline data-based ranking metrics like AUC and NDCG, which can produce offline performance inconsistent with actual online performance. 

\begin{figure*}[ht]
\centering
\includegraphics[scale=0.36]{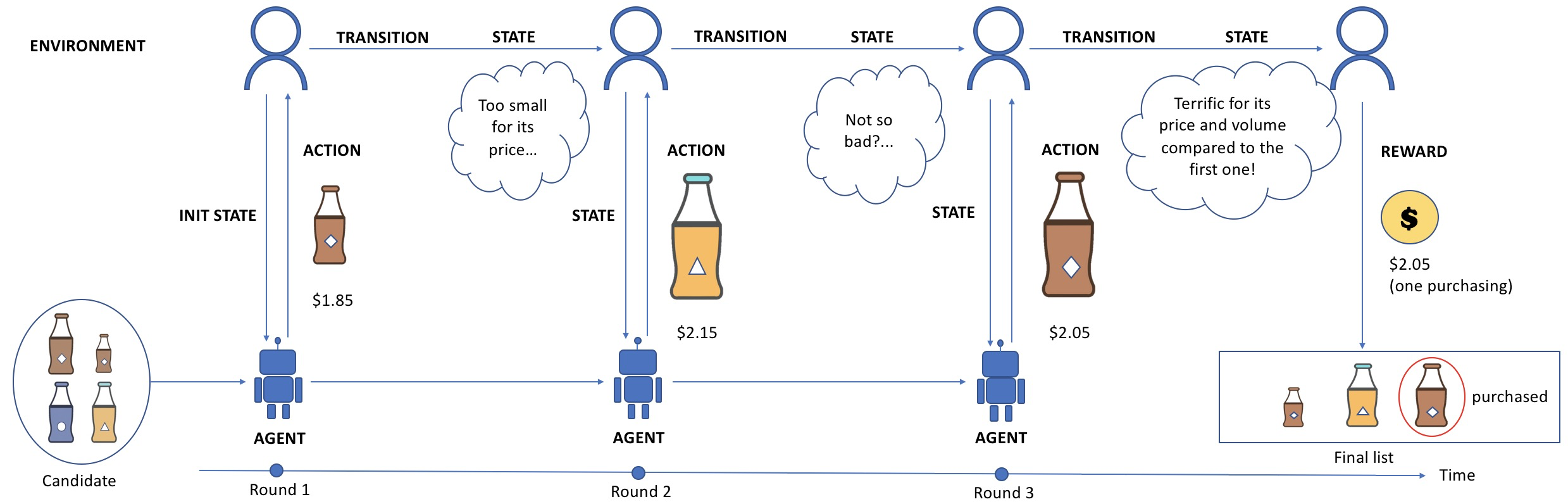}
\caption{Illustration of the Markov decision process modeling of E-commerce interactions. The environment state (schematically) describes the intention of a customer at that time, which is influenced by previous interactions.}
\label{fig:rlframework}
\end{figure*}
Slate optimization is a close topic with LTR. Similar to the objective of re-ranking, it also aims to optimize the profit of the whole slate (i.e. a list or a webpage). Recently, there are a few slate optimization studies pay more attention to metrics that consider the new orders. 
An idea is to combine a generation framework with evaluations on lists to solve slate optimizations. However, multiple sampling~\cite{wang19baidu} and heuristic search (such as beam search ~\cite{Zhuang18}) involved in such approaches can be time-consuming for online application. An independent work~\cite{ZhangMLZ0MXT19} studies a similar Evaluator-Generator reinforcement learning framework in pure offline environments for CTR predictions, which is a secondary goal in E-commerce. A solution for exact-K recommendation~\cite{gong19exact} is to imitate outputs with positive feedback through behavior cloning. Compared with such approaches, EG-Rerank+ encourages models to over-perform the experts (offline data) through generative adversarial imitation learning (GAIL)~\cite{goodfellow15gan}, which has been examined to be a better choice for imitation learning ~\cite{ho16gail,finn2016guided,shi2019virtual}. A team from Huawei~\cite{wu2018turning} proposes an alternative offline metric for GMV optimization models.

Reinforcement learning~\cite{sutton1998introduction} algorithms aim to find the optimal policy that maximizes the expected return. Proximal policy optimization (PPO)~\cite{schulman2017proximal} is one of them, which optimizes a surrogate objective function by using the stochastic gradient ascent. PPO retains some benefits of region policy optimization (TRPO)~\cite{schulman2015trust}, but it is much simpler to implement with better sample complexity. Some works~\cite{slateQ,topkoff} use a pure reinforcement learning method for slate optimization and achieved good performance in online environments. Their methods focus on a series of slates and optimize long term value (LTV) of user feedback. Different from their works, we focus on the fundamental challenge to optimize a single slate with only one round of interaction.

The offline-online inconsistency problem has been reported by other studies~\cite{rohde2018recogym,beel2013comparative,rossetti2016contrasting,mcnee2006being}. Unbiased learning is a related topic to mitigate selection bias. These solutions (e.g., Propensity SVM-Rank~\cite{joachims2017unbiased}, SVM PropDCG and Deep PropDCG~\cite{agarwal2019general}) can increase the accuracy of trained models. Different from our focus, these models do not consider context and do not attempt to find the best order.


\begin{figure*}[t]
\centering
\includegraphics[scale=0.54]{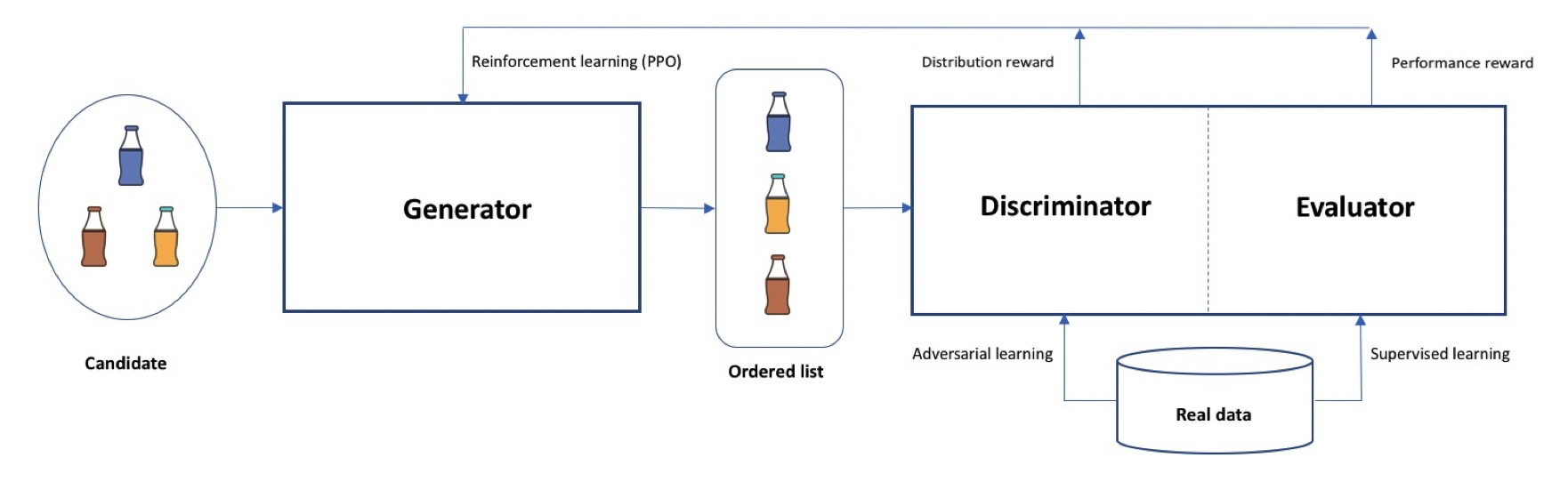}
\vspace{-1em}

\caption{{The architecture of EG-Rerank+}. For EG-Rerank, there is no discriminator module. The evaluator is trained and fixed first, and then we can train the generator with rewards provided by the evaluator. The generator and the discriminator are trained simultaneously.}
\label{fig:framework}
\end{figure*}
\section{Preliminary}

\subsection{Problem Definition}
The objective of a re-ranking model is to find the best permutation of $N$ candidate items. In this paper, we assume that customers browse the list of items displayed in a webpage from top to bottom. Let $o_{1:i}$ denote the arrangement of top $i$ items. When a customer $u$ looks at a list $o$, the probability that customer $u$ purchases the $i$-th item can be expressed as $p(C_i|o_{1:i}, u)$, where \emph{the random variable $C_i$} denotes the event that customer $u$ purchases the $i$-th item given the layout $o_{1:i}$. Classic LTR models treat $p(C_i|o_{1:i}, u)$ as $p(C_i| u)$, where the items have no mutual influence. Different from these models, re-ranking models have complete knowledge of candidate items.

We use $C$ to denote the number of purchases for the whole list (i.e. the summation of $C_i$). Then, we have a formula that computes the expected number of purchases for list $o$ and customer $u$ as follows:
\begin{equation}
\mathbb{E}(C|o, u) = \sum_{i = 1} ^ N \mathbb{E}(C_i|o, u) = \sum_{i = 1} ^ N p(C_i|o_{1:i}, u) 
\label{eq:tar}
\end{equation}

Given a fixed set of candidate items $I$ and a fixed customer $u$, the goal is to find a permutation $o^*$ that maximizes the expected number of purchases:
\begin{equation}
o^* = \operatorname*{argmax}_{o \in perm(I)} \mathbb{E}(C|o, u) = \operatorname*{argmax}_{o \in perm(I)}\sum_{i = 1} ^ N p(C_i|o_{1:i}, u) 
\end{equation}
where $perm(I)$ is the set of all permutations of set $I$.

\subsection{Reinforcement Learning Re-ranking}
\label{section:rldef}
For an ordered list of items, we assume that 
the customer browses each item from top to bottom and decides whether to purchase or not. Given a candidate set $I$ containing $N$ items and the customer $u$, we model a re-ranking task as a Markov Decision Process (MDP) with a tuple of $\langle S, A, P, R\rangle$ with discount rate $\gamma=1$ (also see Figure~\ref{fig:rlframework}):

\begin{itemize}
	\item State space $S$: a state $s_t \in S$ consists of the user feature and an ordered list of selected items before time $t$.
	Concretely, we have $s_t = (u, o_t)$ and the observation $o_t = (x_0, x_1, ..., x_k, ..., x_{t-1})$, where $x_k$ is the item selected from the candidate set $I$ at time $k$. Specially, the initial list $o_0 = ()$ is an empty list.
	
	\item Action space $A$: an action $a_t \in A$ consists of a single item $x_t$ which is selected from the items set $I \setminus \{x_0, x_1, ..., x_{t-1}\}$.
	
	\item Reward $R$: a reward $r_t$ can be written as $r_t = R(s_t, a_t) = R(s_t, x_t) = p(x_t |s_t) $, where $p(x_t|s_t)$ indicates the probability that the customer purchase the item $x_t$. 
	
	\item Transition probabilities $P$: After the action $a_t$ has been chosen (i.e. $x_t$), the state deterministically transitions to the next state $s_{t+1}$, where $s_{t+1} = (u, o_{t+1}) = (u, (x_0, x_1, ..., x_{t-1}, x_t))$ so that $P_{s_t, s_{t+1}} = 1$ (and $0$ for others). 
	
\end{itemize}

In our work, we design an evaluator to estimate the purchase probability $p(x_t|s_t)$. The evaluator plays an important role in training, and can be learned by any discriminative machine learning method. The reward, which is given by the evaluator, helps train the generator through reinforcement learning. Models trained by supervised learning can select the short-term best action in each step, such as finding the item that maximizes purchase probability $p(x_t|s_t)$, but the greedy strategy cannot optimize the long-term reward.

\subsection{LTR in Online Systems}
A widely accepted view is that people pay more attention to items which appear early in the list. Therefore, it makes sense to design a model to find the best-matched items and put them at the top of a list. 
Following this guideline, we can reduce the ranking task to a conversion rate prediction task as most LTR models do. 
However, the conversion rate depends on the context in practical applications. A study~\cite{wang19baidu} has shown when a customer clicks an item (news in their case), the chance of clicking the next item decreases. In situations in which people never click adjacent items, greedy ranking approaches produce improper orders. 
The above example shows that even an accurate model with high AUC and NDCG scores can perform poorly in some scenarios due to contextual factors. Therefore, \emph{comparing models by data-based metrics in a static offline dataset may be misleading}. It is desirable for E-Commerce LTR research to develop a more robust evaluation for examining models.

\section{The Proposed Approach}
We propose to use an evaluator-generator framework that trains the generator of lists not from fixed labeled data or optimizing data-based metrics, but learn to generate lists that maximize the evaluator scores.

In this framework, the evaluator needs to evaluate the performance of any given list, and it is expected to closely track the actual online performance. To train a list generator that maximizes the evaluator score, a natural tool is to employ reinforcement learning to transfer gradient information. As a result, we propose the evaluator-generator re-ranking approach (EG-Rerank). Moreover, to ensure that the evaluator gives trustable scores, we further introduce a discriminator to tell whether the generated list is far from the data. This results in the implemented version, EG-Rerank+, contains an evaluator, a generator, and a discriminator:
\begin{itemize}
  \item The trained evaluator predicts the performance of the given lists. We use a supervised learning approach.
  \item The trained generator produces orders with high scores (from the evaluator). We propose a reinforcement learning approach to train it with rewards provided by the evaluator and the discriminator. 
  \item The trained discriminator measures how much the predictions of the evaluator can be trusted. We design an adversarial learning approach with the outputs of the generator and labeled data. Without a discriminator, the evaluator might wrongly evaluate the performance of lists which are much different from the training data.
\end{itemize}

Figure~\ref{fig:framework} shows the architecture of EG-Rerank+. Details 
is described in the following subsections.


\begin{figure*}[t]
\centering
\includegraphics[scale=0.321]{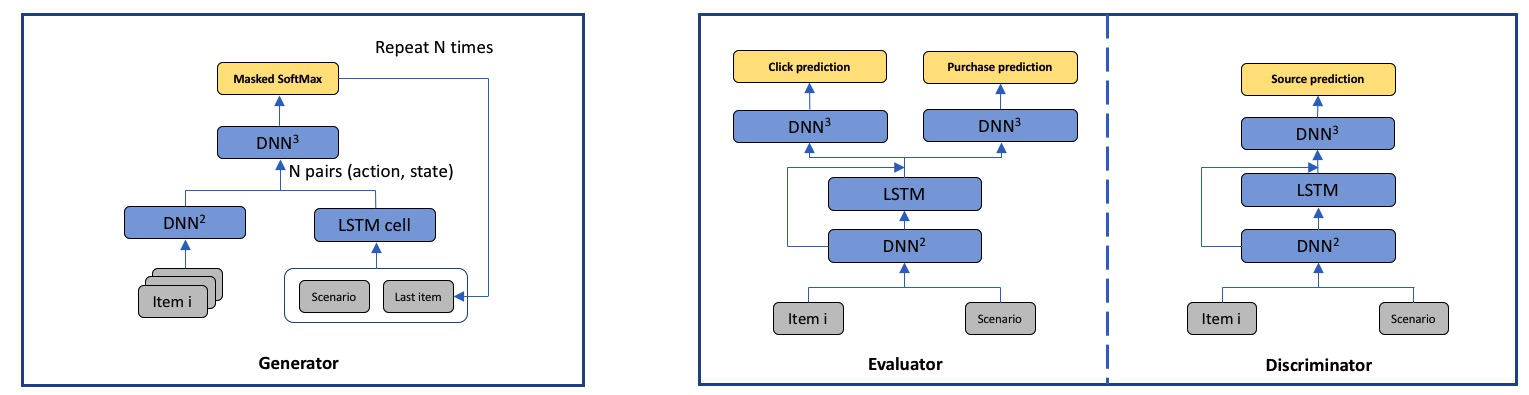}
\caption{The network structure of the generator, the evaluator and the discriminator.}
\label{fig:framework2}
\end{figure*}

\subsection{The Evaluator}
The evaluator is the key model in EG-rerank. It produces the score, or the reward, for training the generator. The structure of our evaluator is shown in Figure~\ref{fig:framework2}. 
The input includes the features of a list of items $(x_1, x_2, ..., x_N)$ and the scenario feature $bg$. The scenario feature is independent from the items, and provides additional information such as date, language, and user profiles. Let $DNN^k$ denote a network with $k$ fully connected layers. We use $DNN^2$ to extract the hidden feature $s_i$ for each item and an LSTM cell to process the contextual state $h_i$ of the first $i$ items: 
\begin{equation}
s_i = DNN^2([x_i, bg]), \quad h_i = LSTM(h_{i-1}, s_i)
\end{equation}
We include another $DNN^3$ to estimate the conversion rate of item $i$ under state $h_i$ and $h_0$ is initialized with encoding of candidate set $I$. In addition, we co-train~\cite{ma2018entire} a click-through rate prediction task to overcome the issue of the sparsity of purchased samples. It helps the model learn common knowledge for predicting clicks and purchases. 
\begin{align}
p_i &= \mathrm{sigmoid}(DNN^3([s_i, h_i])) \\
p^{click}_i &= \mathrm{sigmoid}(DNN^3_{click}([s_i, h_i]))
\end{align}
The loss function is a weighted sum of both objectives. The parameter $\alpha$ should be chosen according to the ratio between the number of purchases and the number of clicks. Let label $y_i$ and $y_i^{click}$ be feedback $\{0, 1\}$, and $\theta$ be the parameters of the model, we have
\begin{equation}
\begin{aligned}
L(\theta) &= \sum \mathrm{cross\_entropy}(p_i, y_i) \\ 
&+ \alpha * \sum \mathrm{cross\_entropy}(p^{click}_i, y^{click}_i) 
\end{aligned}
\end{equation}

{\bf Requirements for the evaluator:}
Different from classic conversion prediction models, we want the evaluator to evaluate the quality of lists that closely reflect online performance. We provide several minimal requirements in our experiments. Based on logs of the industrial ranking system, the evaluator should satisfy the following requirements:
\begin{itemize}
  \item It should be able to identify the better one between a logged list and the reversed list.
  \item It should be able to identify the better one between a logged list and a random list with the same items.
  \item It should be able to identify the better one with high confidence between two lists with different labels (e.g., one has a purchase and the other does not).
\end{itemize}

In our online experiment, the proposed evaluator achieves the first two design requirements with closely to accuracy of 100\%. For the third one, it achieves an accuracy of $0.8734$, implying that it can identify the better list with high confidence. More details will be discussed in the experimental evaluation section.

\subsection{The Generator}
With a reliable evaluator, the generator can autonomously explore for the best order. 
To encode inputs to the generator, we partition the computation into two parts: the feature of the current state and the feature of available actions:
 \begin{equation}
 \begin{aligned}
 h_{s}^{(t)} &= LSTM(h_{s}^{(t-1)}, DNN^2([bg, x_{out_{t-1}}]))\\ h_{a}^{(i)}  &= DNN^2(x_i)
 \end{aligned}
 \end{equation}
where $x_{out_{t-1}}$ is the item picked by the user in step $t-1$. Note that $h_{a}^{(i)}$ is independent from $t$ and can be reused. The output of the encoding process contains $N$ vectors, where $enc_i^{(t)}$ is a combination of a candidate item and a state. 
 \begin{equation}
 enc^{(t)} = [enc_i^{(t)}]_{i=1}^{n} = [h_{a}^{(i)}; h_{s}^{(t)}]_{i=1}^{n}
 \end{equation}
Then, the generator samples the next action according to $softmax(DNN^3(enc^{(t)}_i))$ for the unpicked item $i$.

{\bf Training algorithm:}
We use a state-of-the-art reinforcement learning PPO~\cite{schulman2017proximal} to optimize the generator with feedback from the evaluator.
As a recent empirical work revealed, estimating the value of state by multiple sampling is beneficial 
in a combinatorial space~\cite{kool2019buy}. 
In our work, we use this strategy to estimate the value and will show its advantage in the experiment section. 
Concretely, we sample $k$ trajectories (i.e. $k$ generated complete lists) $\tau_{i}^{s_t}$ with the current policy start from a state $s_t$. By these trajectories, we can calculate the estimate of state value, denoted as $\widetilde{V}(s_t)$. Other notations can be found in Subsection~\ref{section:rldef}.
\begin{equation}
\widetilde{V}(s_t) = \frac{1}{k}\sum_{i=1}^{k}{\sum_{(s, a) \in \tau_{i}^{s_t}}R(s, a)}
\end{equation}
Here $\sum_{(s, a) \in \tau_{i}^{s_t}}R(s, a)$ is the return, which is also the total number of purchase from the $t$-th item to the $n$-th item. Moreover, we apply the standard deviation of value estimation to the loss function to make training more stable. The standard deviation of $\widetilde{V}(s)$ can be formulated as 
\begin{equation}
\sigma_{\widetilde{V}}(s_t) = \sqrt{\frac{1}{k}\sum_{i=1}^{k}{\left(\sum_{(s, a) \in \tau_{i}^{s_t}}{R(s, a)}-\widetilde{V}(s)\right)^2}}
\end{equation}
The loss function of EG-Rerank is written as 
\begin{equation}
L(\theta) = -\hat{\mathbb{E}}_{t}[\text{min}(r_t(\theta)\hat{A}_t, \text{clip}(r_t(\theta), 1-\epsilon, 1+\epsilon)\hat{A}_t)] 
\end{equation}

\begin{equation}
\hat{A}_t = \frac{\sum_{i=t}^{N}{R(s_i, a_i)} - \widetilde{V}(s_t)}{\sigma_{\widetilde{V}}(s_t)}
\end{equation}

\begin{equation}
r_t(\theta)= \frac{\pi_\theta(a_t|s_t)}{\pi_{\theta_{old}}(a_t|s_t)}
\end{equation}
Here, $clip(x, l, r)$ is equivalent to $min(max(x, l), r)$. Policy $\pi_{\theta_{old}}$ is the one collecting rewards and $\pi_\theta$ is the current policy. Incorporating $r_t(\theta)$ protects the model from unstable policy changes. A generator generates an order by sampling a trajectory and send it to the evaluator to obtain rewards. Then it updates its parameters with the above loss function.
\subsection{EG-Rerank+}

All supervised LTR models have a commonly issue: the data we use to train and validate the model is from the offline data, which is collected by a specific system and is generally has a significant bias on selecting samples. It implies that the model has not been well trained in the unseen lists, so it may not give a correct prediction on the unseen lists.


As our evaluator follows the supervised learning paradigm, it also has the above risk. To visualize this issue, we set up a toy task: finding the best permutation of $30$ items. A simple rule determines the score of permutations, and the score completely depends on the order of items. We train a model to regress this score with a training set, which is collected by a random but biased strategy. At the same time, we collect a biased validation set. After that, we compute the prediction error of the model in the biased validation set and the unbiased test dataset (which is collected by a random strategy). 

\begin{figure}[h]
\centering
\includegraphics[scale=0.246]{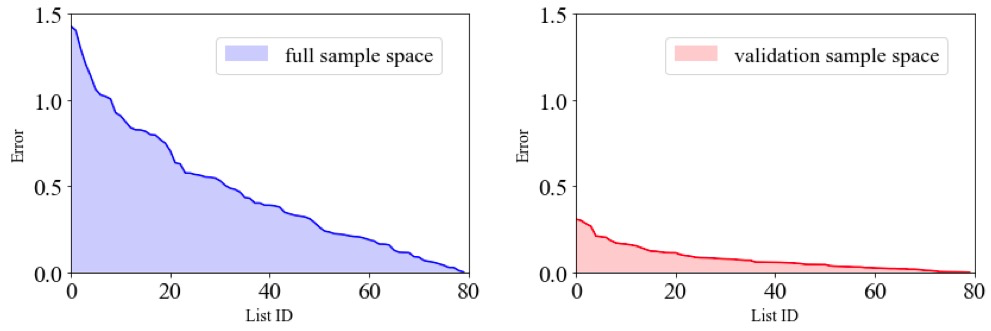}
\caption{The error (difference between prediction and ground-truth) in the full sample space (left) and the validation set (right) in a simulation environment.
The list with a smaller ID should have a greater error as we sorted them in a decreasing order of error.}
\label{fig:ebias}
\end{figure}

The above experiment tries to reproduce the inaccuracy of an evaluator when the input distribution is different from the lists in the training set. 
For better virtualization, we arrange lists in \emph{a decreasing order} of error as the Figure~\ref{fig:ebias} shows. It is clear to see the error is higher in the unseen distribution.

\subsubsection{Design of EG-Rerank+}
To ensure that the generator closely mimics the real lists, we introduce a sequential discriminator. The designed discriminator prevents the generator from outputting lists which are too different from the training data (i.e. when the evaluator may provide wrong supervision). With the guidance of the discriminator, the generator will explore solutions in a more reliable and regular space. 

Discriminator score $D(x|w)$ represents how a list is possibly from real data as judged by the discriminator. Formally, for an ordered item list $x = (x_1,x_2,...,x_N)$, the discriminator will score each item, and we let $D(x|w)$ be the summation of scores of items. A sequential structure produces $D(x|w)$ as follows:
\begin{equation}
  s_i = DNN^2([x_i, bg])
\end{equation}
\begin{equation}
  h_i = LSTM(h_{i-1}, s_i)
\end{equation}
\begin{equation}
  score_i=DNN^3([s_i, h_i])
\end{equation}
\begin{equation}
  D(x|w) = \sum{score_i}
\end{equation}
We aim to train the discriminator to distinguish generated lists $x$ from real lists $x'$. We want to maximize the below expectation and the gradient has the form
\begin{equation}
\hat{\mathbb{E}}_{x}\left[\nabla_{w} \log \left(D(x|w)\right)\right]+\hat{\mathbb{E}}_{x'}\left[\nabla_{w} \log \left(1-D(x'|w)\right)\right]
\end{equation}
Finally, we take the output of the discriminator as part of the reward for learning under EG-Rerank+ with an adjustable parameter $\beta$:
\begin{equation}
R^+(s_i, a_i) = R(s_i, a_i) + \beta * score_i
\end{equation}

\subsubsection{Advantages of EG-Rerank+} 
To reveal the advantage of EG-Rerank+, we set up an auxiliary experiment by real data. 
We collect thousands of real lists from our online system, where all lists are generated from the same search query ``phone screen protectors''. Figure~\ref{fig:ebias2} visualizes the distribution of real lists, outputs of EG-Rerank, and outputs of EG-Rerank+. 
We plot Figure~\ref{fig:ebias2} with t-SNE for data dimensionality reduction, in which the sources (i.e. generated from which model) are not revealed. To reduce the noise from the online environment, we remove $20\%$ records which are farthest away from the centroid in their groups.

\begin{figure}[h]
\centering
\includegraphics[scale=0.18]{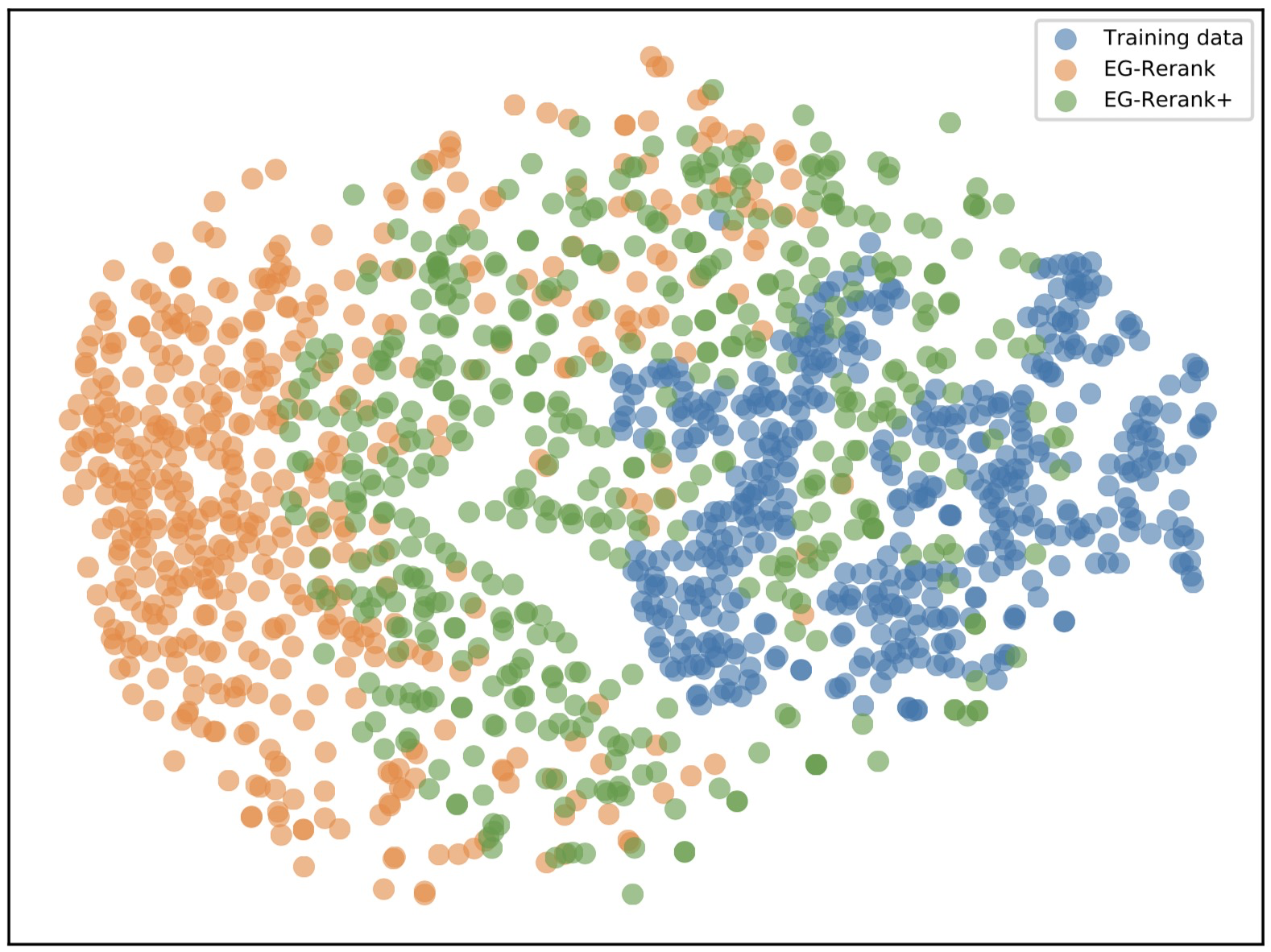}
\caption{{The distribution of real lists and generated lists}. }
\label{fig:ebias2}
\end{figure}

We can observe that outputs of EG-Rerank+ are closer to the real lists produced by a mature system than the outputs of EG-Rerank. 
The discriminator idea is a strategy to guide the exploration without having to engage in costly online interactions. In practice, it is possible to deploy online exploration strategies to calibrate the evaluator or the generator. Compare to these methods, EG-Rerank+ is easy to implement, poses no risk to customer experience, and can be examined in offline environments.

\begin{figure*}[ht]
\centering
\includegraphics[scale=0.312]{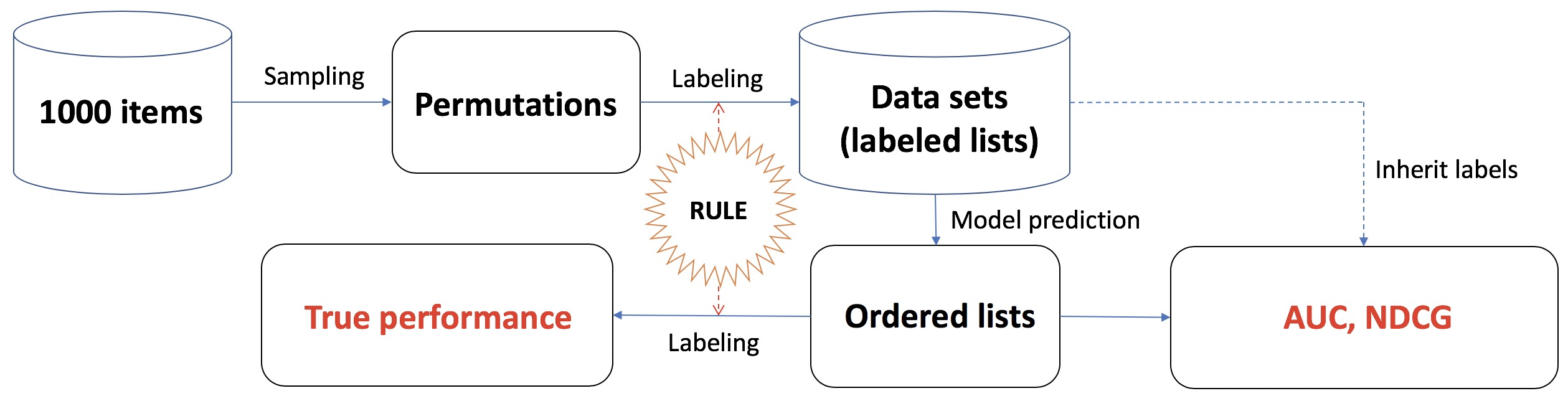}
\caption{{The workflow of the simulation environment}. An extra interaction with simulator RULE is necessary to get the true performance.}
\label{fig:egframe}
\end{figure*}

\section{Experimental Evaluation}


For re-ranking tasks, an evaluation based on implicit feedback from existing datasets may lead the inconsistency problem: offline data can hardly reflect the accurate feedback in an unseen distribution. Therefore, a \emph{dynamic judgement} which can produce feedback for every possible output from models needs to be introduced. We consider the two examinations that satisfy the \emph{dynamic} condition:
\begin{itemize}
  \item \emph{Examination in the simulation environment}. To the best of our knowledge, it is the only offline choice to roughly evaluate model performance. With well-defined simulation protocols, the experimental results can be easily reproduced.
  \item \emph{Examination in AliExpress Search online system}. Although online trials are expensive and not reproducible, online A/B testing is the only gold standard of evaluation.
\end{itemize}
In the following contents, we conduct our models in both of them and analyze two research questions. \textbf{(RQ1)} Does the evaluation in the static dataset evaluation mislead the learning? \textbf{(RQ2)} Does EG-Rerank+ achieve the best performance among existing re-ranking algorithms?




\subsection{Experiment in a Simulation Environment}

We demonstrate that classic data-based metrics are inconsistent with our objective (e.g. the number of conversion) even in a simple LTR scenario with mutual influence information between items. The offline experiments is conducted in a simulation environment\footnote{The code of the environment are released in github https://github.com/gaoyq95/RerankSim}.


\subsubsection{The Simulation Environment}
Our simulation environment borrows ideas from the design of RecSim~\cite{ie2019recsim}. We prepare $1,000$ items each containing a random feature of length $30$, and sample $400,000$ lists of size $15$ uniformly from the $1,000$ items. The simulator, which acts as referee in the environment, labels these sampled lists. Then, we divide these lists into two parts of $300,000$ and $100,000$ as the training data and the testing data. To get the true performance, we need an extra interaction with the simulator as Figure~\ref{fig:egframe} shows.

Let $L$ be a list and $L_i$ be the $i$-th item in list $L$. Each item receives two scores from the simulator (the ``RULE'' in Figure~\ref{fig:egframe}):
\begin{equation}
  f(L_i,L) = \alpha_i r(L_i) + \beta_i g(L_{1:i})
\end{equation}
The base conversion rate $r(L_i)\in[0,1]$ of item $L_i$ is computed using a randomly weighted DNN. For mutual influence $g(L_{1:i})$, we consider the similarity between features of selected items. Intuitively, it is better to avoid adjacently displaying similar items. In our experiment, mutual influence $g$ deploys $1$ minus cosine distance in the range of $[0, 1]$, between $L_i$ and the average of $L_{1:i}$. In addition, 
we let constant $\alpha_i$ decrease and constant $\beta_i$ increase according to position $i$ and $\alpha_i + \beta_i = 1$. Therefore, score $f(L_i, L)$ must be a real number in $[0, 1]$ and each item receives a $\{0, 1\}$ label by sampling with the probability $f(L_i, L)$. We regard the label $1$ as a purchase event on an item. 

The score of a list, which is the summation of scores of all items, equals the expected number of purchases. We denote this score as the \emph{true score} on a list, and regard the true score as the main objective of a model. Besides, we use labels (collected by old orders) to compare the accuracy of data-based metrics. 

\begin{table*}[ht]
\centering
\small
\resizebox{.99\textwidth}{!}{
\begin{tabular}{cccccc}
\toprule
 Method groups& Method& GAUC &NDCG & Evaluator score & True score\\ 
\toprule

\multirow{3}{*}{Non-deep methods} 
 & SVM-Rank & 0.93948$\pm$0.00002&0.96826$\pm$0.00002&5.56269$\pm$0.00004&5.57967$\pm$0.00011 \\
 & Prop SVM-Rank & 0.93661$\pm$0.00001&0.97079$\pm$0.00001&5.56684$\pm$0.00001&5.58434$\pm$0.00001 \\
 & LambdaMART & 0.93188$\pm$0.00154&{0.97697$\pm$0.00106}&5.58462$\pm$0.00765&5.59179$\pm$0.00827 \\
 \midrule[0.1pt]

\multirow{3}{*}{Point-wise methods} & miDNN (MSE loss) &0.94767$\pm$0.00006&0.97266$\pm$0.00007&5.59288$\pm$0.00230&5.53829$\pm$0.00304 \\
 & miDNN (CE loss) & 0.94764$\pm$0.00006&0.97282$\pm$0.00004&5.61278$\pm$0.00462&5.56168$\pm$0.00572\\
 & miDNN (Hinge loss) & 0.93957$\pm$0.00114&0.96863$\pm$0.00063&5.56862$\pm$0.00826&5.53250$\pm$0.00616\\
 \midrule[0.1pt]

\multirow{2}{*}{Pair-wise methods} & RankNet (Logistic loss) &0.94724$\pm$0.00002&0.97249$\pm$0.00002&5.64835$\pm$0.00153&5.59949$\pm$0.00110 \\
 & RankNet* (Hinge loss) & {0.94813$\pm$0.00003}&0.97302$\pm$0.00003&5.59475$\pm$0.00355&5.54473$\pm$0.00291\\
 \midrule[0.1pt]

\multirow{3}{*}{List-wise methods} & ListNet & 0.94770$\pm$0.00004&0.97279$\pm$0.00003&5.62637$\pm$0.00374&5.57817$\pm$0.00363 \\
& ListMLE & 0.94712$\pm$0.00004&0.97261$\pm$0.00003&5.60474$\pm$0.00119&5.56719$\pm$0.00155\\
& SoftRank & 0.94728$\pm$0.00007&0.97270$\pm$0.00004&5.60824$\pm$0.00355&5.56788$\pm$0.00101 \\
 \midrule[0.1pt]
\multirow{2}{*}{Group-wise methods} & GSF(5) & 0.94792$\pm$0.00003&0.97286$\pm$0.00001&5.61773$\pm$0.00221&5.56135$\pm$0.00255 \\
& GSF(10) & 0.94807$\pm$0.00003&0.97282$\pm$0.00004&5.61838$\pm$0.00311&5.56526$\pm$0.00335 \\
 \midrule[0.1pt]

\multirow{2}{*}{Advanced methods}
& seq2slate & 0.64042$\pm$0.00063&0.75267$\pm$0.00057&5.63576$\pm$0.00326&5.62322$\pm$0.00358\\
& PRM & 0.94787$\pm$0.00003&0.97283$\pm$0.00002&5.78159$\pm$0.00050&5.71738$\pm$0.00068\\
 \midrule[1pt]

\multirow{9}{*}{Evaluator-Generator}
& DirectE  &0.96365$\pm$0.00029&0.98099$\pm$0.00014&5.63850$\pm$0.00196&5.59867$\pm$0.00228
\\
& GreedyE & 0.92918$\pm$0.00081&0.96388$\pm$0.00045&5.77470$\pm$0.00353&5.73948$\pm$0.00327
\\
& DQN (SlateQ) & 0.71726$\pm$0.03948 &0.79007$\pm$0.01662&6.18838$\pm$0.07606 &6.19378$\pm$0.05797
\\
& Monte Carlo Control&0.77541$\pm$0.00957 & 0.81201$\pm$0.01019 & 6.36602$\pm$0.06473 & 6.29952$\pm$0.04813
\\
& CTR-AC& 0.56332$\pm$0.07063 	& 0.72339$\pm$0.04117 	& 5.61957$\pm$0.05035 	& 5.76285$\pm$0.04597\\
& CTR-AC+ &0.48923$\pm$0.05965	&0.65473$\pm$0.02030 	& 7.03191$\pm$0.03798 	& 6.71545$\pm$0.02517\\
& PPO& 0.74747$\pm$0.00691 	& 0.76107$\pm$0.00530 	& 7.10018$\pm$0.02135 	& 6.78382$\pm$0.01566\\
& \textbf{EG-Rerank}  & 0.55774$\pm$0.00442&0.69925$\pm$0.00279&7.36499$\pm$0.01294&6.96224$\pm$0.00972\\
& \textbf{EG-Rerank+}  & 0.51376$\pm$0.01181&0.69171$\pm$0.00338&\textbf{7.36847$\pm$0.01341}&\textbf{6.97283$\pm$0.01240} \\ 
\bottomrule
\end{tabular}
}
\vspace{0.8em}
\caption{Models performance in the simulation environment.}
\vspace{-1.6em}

\label{tab:offlineexp}
\end{table*}

In a real environment, top-ranked items receive more attention from customers. We calculate the average conversion rate for each position in our simulation environment with training data. The result is that the conversion rate decreases from $0.5$ to $0.3$ as the position increases.
The property motivate us to display the best items in the top as we do in online systems. However, most of previous simulation experiments (such as \cite{jiang19cave}, \cite{ie2019recsim}, and \cite{rohde2018recogym}) did not promise the above important condition. It brings unfairness for ranking methods which place the best items in the top.

\subsubsection{Comparison Baselines}
We arrange the models in several groups:
\begin{itemize}
  \item \textbf{Non-deep methods}. We use SVM-Rank~\cite{joachims2017unbiased} and LambdaMART~\cite{burges2010ranknet} to represent SVM and boosting LTR methods. We further add Propensity SVM-Rank~\cite{joachims2017unbiased} to examine benefits of de-biasing,.
  \item \textbf{Point-wise methods}. We use miDNN~\cite{Zhuang18} and apply mean square root error (MSE) loss, cross entropy (CE) loss and hinge loss on it.
  \item \textbf{Pair-wise methods}. We apply pair-wise logistic loss and hinge loss on miDNN to represent pair-wise methods which follow RankNet~\cite{Burges:ranknet}. RankNet* is the pair-wise method we deployed online. 
  \item \textbf{List-wise methods}. We include ListNet ~\cite{cao07list}, ListMLE~\cite{ai2018learning}, and SoftRank~\cite{taylor2008softrank} to demonstrate the performance of list-wise approachings. 
  \item \textbf{Group-wise methods}. We examine $GSF(5)$ and $GSF(10)$ with the sampling trick introduced in the group-wise scoring framework paper~\cite{ai2019learning} .
  \item \textbf{Advanced Re-ranking}. We add a pointer network solution seq2slate~\cite{vinyals2015pointer,seq2slate} with cross entropy loss into our experiment, and a industial re-ranking method PRM with a random initialization~\cite{pei2019personalized}.
  
  \item \textbf{Evaluator-Generator}. GreedyE is a greedy strategy to pick the item which can produce the maximal evaluator score immediately. DirectE uses the evaluator as a classic LTR model to generate lists. Deep Q Network (DQN) is a traditional Q-learning algorithm, and Monte Carlo Control is a substitute with samplings which brings less bias and more computations. DQN can be regarded as a degeneration version of SlateQ~\cite{slateQ} with $k=1$ without involving a CTR prediction (in our testing, it is harmful to model performance). CTR-AC is a recent work that follows REINFORCE~\cite{ZhangMLZ0MXT19}. To enhance it, we apply the advantage function to CTR-AC as CTR-AC+.
  EG-Rerank and EG-Rerank+ (our solutions), the advanced function are computed by sampling instead of using $V$ model in PPO.
  
  \item \textbf{Evaluators}. Note \emph{all involved evaluators are trained by offline data}, and play as an environment in the training phase of the above methods.
\end{itemize}

\subsubsection{Result analysis}
We take four indicators (NDCG, GAUC, Evaluator Score and True score) into consideration:
\begin{itemize}
  \item \emph{GAUC (offline) and NDCG} are two commonly used ranking metrics.
  \item \emph{Evaluator score} is the prediction from the evaluator.
  \item \emph{True score} is the goal that models aim to maximize. It is produced by the simulator as we described.
\end{itemize} The experiments are repeated $10$ times with independent data. We display the mean performance and standard deviation in the Table~\ref{tab:offlineexp}.

{\textbf{Metric accuracy:}} In the above table, GAUC and NDCG are consistent with each other, but they are inconsistent with the evaluator score as well as the actual performance. Concretely, we enumerate each pair of metrics and compute normalized Kendall Tau Distance as Table~\ref{tab:cor} shows.
\begin{table}[h]
\centering
\small
\resizebox{.45\textwidth}{!}{

\begin{tabular}{c|cccc}
\toprule
 & GAUC & NDCG & Evaluator score& True score\\
\midrule
GAUC & - & \textbf{0.084} & 0.645 & 0.718 \\
NDCG&\textbf{0.084} & - & 0.658 & 0.730 \\
Evaluator score&0.645 &0.658 & - & \textbf{0.099} \\
True score&0.718 &0.730 & \textbf{0.099} & - \\
\bottomrule
\end{tabular}}
\vspace{0.8em}
\caption{Normalized Kendall Tau Distance results for each pair of metrics.}
\vspace{-1.2em}
\label{tab:cor}
\end{table}

The result shows that GAUC and NDCG can only predict the better method between two candidates with a poor probability (less than $30\%$) by the definition of Normalized Kendall Tau Distance. Instead of that, evaluator can correctly decide the better method with a high probability (greater than $90\%$). This fact is again provide a strong evidence to \textbf{RQ1} that data-based metrics may mislead the learning, and supports that evaluator score is potentially a proper metric for E-commerce LTR model examination.

{\textbf{Performance of Supervised Learning Methods:}} Classic supervised learning methods get high NDCG and GAUC (more than 0.9) in the simulation environment, which implies they can accurately recover the labels in offline data. However, they have poor true score. 
In contrast, most of methods with the Evaluator-Generator framework achieve lower NDCG and GAUC, but output much better lists according to the true score metric. This fact is also a positive example for \textbf{RQ1}.

{\textbf{Performance of Advanced Methods:}}
Unbiased methods like propensity SVM-Rank, as well as sophisticated re-ranking methods such as seq2slate and PRM, can bring a small improvement on the true score, but it is not competitive with methods follow Evaluator-Generator frameworks. We also include reinforcement learning solutions such as CTR-AC to show our competitiveness. Therefore, together with the above paragraph, we partially answer \textbf{RQ2}: in our simulation experiment, EG-Rerank+ outperforms existing re-ranking methods. On the other hand, traditional $Q(s, a)$ and $V(s)$ are difficult to estimate in this ranking task. The comparison between DQN and Monte Carlo Control, as well as the comparison between PPO and EG-Rerank, implies that replacing the value function with sampling can greatly improve the performance. Benefiting from PPO and sampling, our EG-Rerank (and also EG-Rerank+) steadily outperforms slateQ, CTR-AC, and their variants. 
Not only in ranking tasks, we conjecture the similar phenomenons exist in other combinatorial optimization scenarios.





\begin{figure*}[ht]
\centering
\includegraphics[scale=0.3]{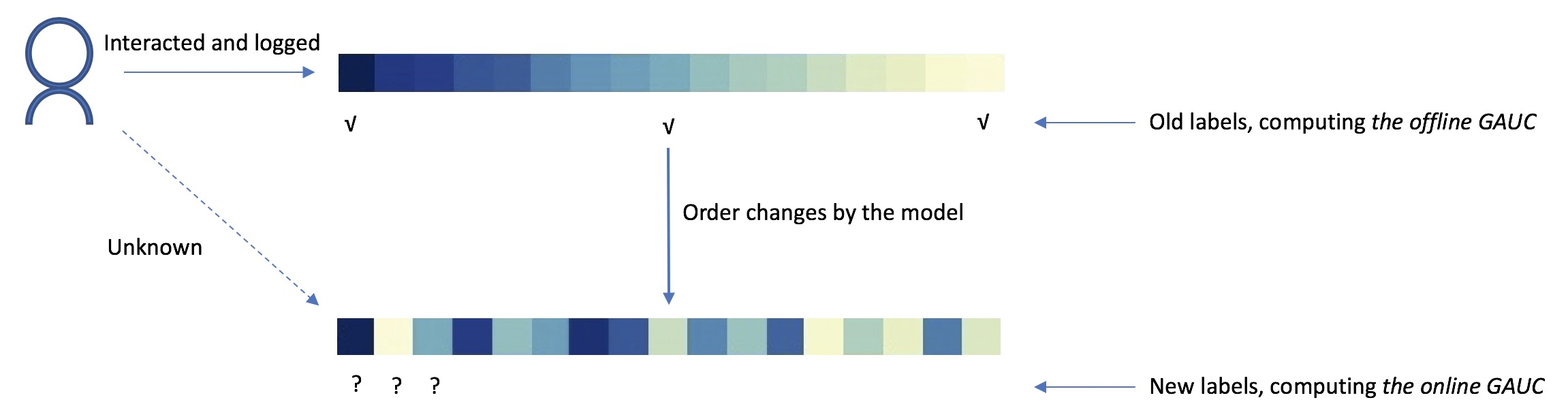}
\caption{The workflows of computing \emph{offline GAUC} and \emph{online GAUC}, respectively. Colored squares appear as items.}
\label{fig:ita}
\end{figure*}

\begin{fact}
Roughly, EG-Rerank+ are expected to output top $0.2\%$ lists in our simulation experiment.
\end{fact}
\begin{proof} 
Consider that we arrange the lists in a interval $[0, 1]$, where adjacent lists have the same spacing. We can use the quantile from $0$ to $1$ (from worst to best) to represent the ranking. Then the ranking of a sampled list can be approximated by the uniform distribution in $[0, 1]$.
From the result table~\ref{tab:offlineexp}, we find that EG-Rerank+ has the similar performance as ENUMERATE-500 does. 
Therefore, we only need to estimate the expected ranking of the above best list as $\mathbb{E}[max(X_1, X_2, ..., X_n)]$,
where $X_i$ is the random variable of the ranking of the $i$-th sampled list. 

Then we can write the expectation in the integral form:
\begin{equation}
\begin{aligned}
  \mathbb{E}[max(\{X_i\})] &= \int_0^1 \mathbb{E}[max(\{X_i\}) = x] dx\\
  &=\int_0^1 x \times p(max(\{X_i\}) = x) dx \\
  &= \int_0^1 x \times nx^{n-1} dx = \frac{n}{n+1}\\
\end{aligned}
\end{equation}
Here probability density function $p(max(\{X_i\}) = x) = \sum_{i=1}^n p(X_i=x,X_{j\neq i}<x) =nx^{n-1}$. Replacing $n$ with 500 into Equation~(21) yields the desired result.
\end{proof}

\begin{table*}[ht]
\centering
\resizebox{.99\textwidth}{!}{
\vspace{0.8em}
\begin{tabular}{c|cccccccccccccc|c}
\toprule
Re-ranking & Day 1& Day 2& Day 3& Day 4& Day 5& Day 6& Day 7& Day 8& Day 9& Day 10& Day 11& Day 12& Day 13& Day 14& Offline GAUC \\
\midrule
RankNet* & 6.11\% & 7.17\% & 8.22\% & \textbf{8.15}\% & 4.19\% & 6.78\% & 7.84\% & 5.95\% & 4.17\% & 5.16\% & 4.18\% & 4.70\% & 3.52\% & 7.34\% & \textbf{0.783}\\
EG-Rerank & \textbf{6.57}\% & \textbf{9.54\%} & \textbf{10.45\%} & 7.98\% & \textbf{5.70\%} & \textbf{10.21\%} & \textbf{10.08\%} & \textbf{9.02\%} & \textbf{7.37\%} & \textbf{6.53\%} & \textbf{6.00\%} & \textbf{7.56\%} & \textbf{6.84\%} & \textbf{10.65}\% & 0.512 \\
\bottomrule
\end{tabular}
}
\vspace{0.8em}
\caption{{Offline GAUC and online performance of models}. Each column ``Day~i'' describe the conversion rate gap after we add a re-ranking method.}
\label{tab:one}

\end{table*}

{\textbf{Expectation Analysis for Searching in the Combinatorial Space:}}
We set up an additional experiment to evaluate the quality of the output list of EG-Rerank+. Consider an inefficient algorithm ENUMERATE-k: we uniformly sample $k$ lists and use the evaluator to decide and output the best list among them. 
\vspace{-1em}
\begin{figure}[h]
\centering
\includegraphics[scale=0.24]{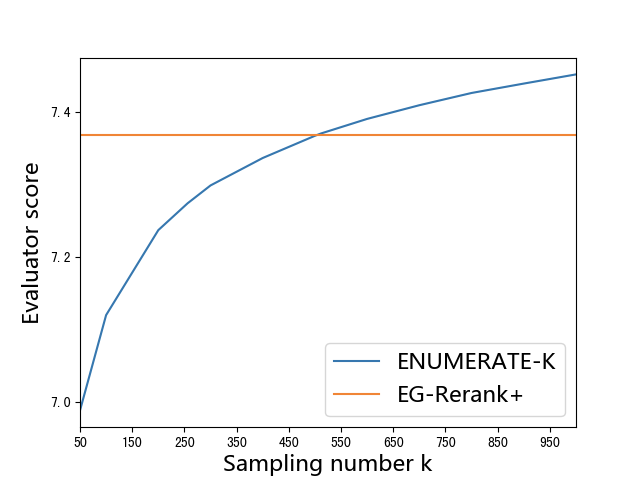}
\caption{Performances of ENUMERATE-k with different k. Due to its low efficiency, we plot the curve by a few piece-wise linear functions.}
\label{fig:egenu}
\end{figure}
Without limitation of time, this algorithm can even generate the global optimal list from the view of evaluator (i.e. enumerate all permutations and select the best one). We plot a curve to find the proper $k$ such that EG-Rerank+ has a similar performance with ENUMERATE-k. Then, we can approximate the combinatorial space searching ability of EG-Rerank+ in expectation. In Figure~\ref{fig:egenu}, we can see EG-Rerank+ and ENUMERATE-500 have almost the same performance.


\section{Online Cases and Experiments}

In the following two subsections, we first share two real cases of AliExpress Search to show the existing inconsistencies between data-based ranking metrics and conversion rate (CR). In our work, CR is the number of purchases divided by the number of arriving customers.
Group AUC~\cite{zhou2018deep} (GAUC) is commonly used in ranking scenarios when scores of items in different lists are not comparable. GAUC only considers the item pairs in the same list (group), and can be computed as $\frac{1}{|L|}\sum_{l\in L} AUC_l$, where $L$ is the set of testing lists.  
In our work, we focus on the purchase events so that only purchased items have positive labels. For a fair comparison, pages with no purchase will not be included for GAUC computing (in E-commerce, there are so many lists without a purchase).
We consider two types of GAUC as follows (also see Figure~\ref{fig:ita}):
\begin{itemize}
  \item Offline GAUC: the one computed \emph{before} a model changes item order (i.e. using the old labels).
  \item Online GAUC: the one computed \emph{after} a model changes item order (i.e. using the new labels).
\end{itemize}

The GAUC computed in training and validation set evaluation, which uses the labels in old orders, is offline GAUC. It is commonly used in the LTR research and applications. Intuitively, online GAUC is an alternative to the offline GAUC. Online GAUC needs the new feedback from users after model changes the order. These two metrics have a gap that should not be ignored \emph{unless we assume the behaviors of customers on items will keep unchanged after the order changes}. In the following section, we will show both of them are problematic when examining models.

\subsection{Inconsistency between Offline GAUC and CR}
\label{section:offauc}
Table ~\ref{tab:one} shows the offline GAUC of two ranking strategies during a week. 
RankNet*\cite{Burges:ranknet} is an industrial-level fine-tuned pair-wise model, and it follows the design of RankNet and has the best online performance in long-term experiments. 

From the Table~\ref{tab:one}, we observe that EG-Rerank achieves poor offline GAUC, but greatly improves the conversion rate by more than $2\%$. 
This is again a strong evidence to answer \textbf{RQ1}
and blindly encouraging models to put the best items on the top (as the first strategy does) also cannot achieve consistently good performance.

\subsection{Inconsistency between Online GAUC and CR}
\label{section:onlauc}

Table~\ref{tab:onlineauc1} shows the online GAUC and actual performance of three deployed strategies in the past. Without knowing them, strategy 3 appears to be the best one according to online GAUC. However, it is problematic and actually the worst strategy due to its low conversion rate. The result clearly shows that online GAUC also cannot reflect actual model performance when deployed online.

\begin{table}[h]
\centering
\begin{tabular}{c|ccccc}
\toprule
 Online GAUC & Day 1 & Day 2 & Day 3 & Day 4\\
\midrule
Strategy 1 & 0.786 & 0.787 & 0.794 & 0.792 \\
Strategy 2 & 0.788 & 0.786 & 0.789 & 0.793  \\
Strategy 3 & 0.805 & 0.803 & 0.800 & 0.797 \\

\toprule
 Conversion Rate Gap &Day 1 & Day 2 & Day 3 & Day 4\\
\midrule
Strategy 1 & 0.00\%	& 0.00\%	& 0.00\%	& 0.00\% & \\
Strategy 2 & 0.72\%	& -0.36\% &	-0.12\% &	-0.49\% & \\
Strategy 3 & -8.45\%	& -8.25\% &	-7.25\%	& -10.13\% & \\

\bottomrule
\end{tabular}
\vspace{0.8em}
\caption{Online GAUC and performances during four days.}
\label{tab:onlineauc1}
\vspace{-1.5em}
\end{table}

The above counter-intuitive phenomenon happens since the i.i.d condition has been violated: different models have different distributions of output lists, which are not compatible at all. Consider such an example: a model always picks the best item from its view, and then adds lots of irrelevant items to the list. In this case, the first item receives a high score and the irrelevant items receive low scores from the model. When no one purchases irrelevant items in such a list, the model achieves high online AUC, MAP, and NDCG, but performs poorly in terms of conversion rate. 


This experiment reinforces the importance of the proposed evaluator: it is not only a module in our framework, but also \emph{a novel metric for evaluating rankings} which more closely reflects online performance than existing metrics and can solve the issue proposed in \textbf{RQ1}.

\subsection{Online A/B tests}
We set up a few online A/B tests on AliExpress Search. All re-ranking methods are deployed when customers use keyword searching, which is also one of the main composition of the flow in our platform. In our long-term trials, fined-tuned RankNet* has been proven to achieve the best online performance among existing re-ranking methods. 
Models can access offline data in the last two weeks for training, and the offline data contains $O(10^8)$ displayed lists and $O(10^6)$ purchase records. In addition, we hold a latency testing for different strategies under a high load condition. The latency results are shown in table~\ref{tab:onlinelatency}.
\begin{table}[h]
\centering
\small
\begin{tabular}{c|ccc}
\toprule
 & No Re-rank & RankNet* & EG-Rerank\\
\midrule
Average latency & 106ms & 113ms & 115ms \\
\bottomrule
\end{tabular}
\vspace{0.8em}
\caption{Testing latency. They are held in the same server.}
\vspace{-1.2em}

\label{tab:onlinelatency}
\end{table}

In each A/B test, two models serve a non-overlapping random portion of search queries. Each model needs to serve $O(10^6)$ users and output $O(10^7)$ pages per day. Since the online environment varies rapidly, to get reliable results, all A/B tests are held for more than a week to reduce the variance. We compute the mean and standard deviation during A/B testing. The result is shown in Table \ref{tab:online}.

\begin{table}[h]
\centering
\small
\resizebox{.48\textwidth}{!}{
\begin{tabular}{c|ccc}
\toprule
 Methods& Online GAUC & Evaluator* & CR gap\\
\midrule
No Re-rank & 0.758 $\pm$ 0.004 & +0.00\% & +0.00\% \\
RankNet* & 0.789 $\pm$ 0.002 & +15.2\% & +5.96\% $\pm$ 0.017  \\
\midrule\
RankNet* & 0.793 $\pm$ 0.002 & +0.00\% & +0.00\% \\
EG-Rerank & 0.783 $\pm$ 0.003 & +9.94\% & +2.22\% $\pm$ 0.011  \\
\midrule

EG-Rerank & 0.774 $\pm$ 0.015 & +0.00\% & +0.00\% \\
EG-Rerank+ & 0.786 $\pm$ 0.003 & +0.81\% & +0.63\% $\pm$ 0.008  \\
\bottomrule
\end{tabular}
}
\vspace{0.8em}
\caption{Online performance. In Conversion Rate Gap columns, the first row is the baseline.}
\vspace{-1.2em}

\label{tab:online}
\end{table}

Here, ``Evaluator*'' denotes the score of the evaluator model which predicts model performance before the A/B tests start, and it is independent with the one used in EG-Rerank.
We can see that the evaluator score is a more consistent and accurate metric than online GAUC. 
The satisfying performance of EG-Rerank+ also provides an positive answer to \textbf{RQ2}. 

\section{Conclusions}
In E-commerce ranking tasks, many studies have lost connections with the real-world, but only focused on offline measurements~\cite{wagstaff2012machine}. To evaluate and improve actual performance, we propose the evaluator-generator framework for E-commerce and EG-Rerank+ for CR optimizations.
We demonstrate the evaluator score is a more objective metric which leads to improvements in online situations.
Through online testing in a large-scale e-commerce platform, the proposed framework has improved the conversion rate by $2\%$, which translates into a significant among of revenue.

\ifCLASSOPTIONcaptionsoff
 \newpage
\fi

\begin{IEEEbiographynophoto}{Guangda Huzhang} received the PhD degree from Nanyang Technological University, Singapore, in 2018. After that, he joined Alibaba and worked as a Machine Learning Engineer up to now.
\end{IEEEbiographynophoto}

\begin{IEEEbiographynophoto}{Zhen-Jia Pang}
received the master degree from Nanjing University, China, in 2020. After that, he joined Huawei and worked as a Machine Learning Engineer up to now.
\end{IEEEbiographynophoto}

\begin{IEEEbiographynophoto}{Yongqing Gao}
is a master's degree candidate in Nanjing University, China. His current research interests include Reinforcement learning and Recommendation system.
\end{IEEEbiographynophoto}

\begin{IEEEbiographynophoto}{Yawen Liu}
is a Master student in Nanjing University, he focuses on Deep Reinforcement Learning, as well as its applictions in Autonomous Driving and Recommendation System.

\end{IEEEbiographynophoto}

\begin{IEEEbiographynophoto}{Weijie Shen} is a Master student in Nanjing University, his research interests include reinforcement learning, imitation learning as well as their applications in Autonomous Driving and Recommendation System.

\end{IEEEbiographynophoto}

\begin{IEEEbiographynophoto}{Wen-Ji Zhou}
received the BSc and MSc degree in computer science from Nanjing University, China, in 2016 and 2019, respectively. He joined the Alibaba-INC in June 2019 and worked as a Machine Learning Engineer up to now.

\end{IEEEbiographynophoto}

\begin{IEEEbiographynophoto}{Qing Da}
received the BSc and MSc degrees in computer science from Nanjing University, China, in 2010 and 2013 respectively. He is currently a senior staff algorithm engineer in the search algorithm team of the Department of International AI at Alibaba Group. His research interests are reinforcement learning and applications of machine learning.\end{IEEEbiographynophoto}

\begin{IEEEbiographynophoto}{An-Xiang Zeng} is a Senior Staff Algorithm Engineer and Director of Alibaba. He is the Head of the International Search and Recommendation of Alibaba. He is pursuing his PhD in Nanyang Technological University, Singapore. He has been working in the search and recommendation field for more than 10 years. His research focuses on search, recommendation and reinforcement learning. He has published more than 10 research papers in leading international conferences and journals.
\end{IEEEbiographynophoto}
\vspace{-0.5em}


\begin{IEEEbiographynophoto}{Han Yu} is a Nanyang Assistant Professor in the School of Computer Science and Engineering, Nanyang Technological University, Singapore. He held the prestigious Lee Kuan Yew Post-Doctoral Fellowship from 2015 to 2018. He obtained his PhD from the School of Computer Science and Engineering, NTU. His research focuses on federated learning and algorithmic fairness. He has published over 150 research papers and book chapters in leading international conferences and journals. He is a co-author of the book \textit{Federated Learning} - the first monograph on the topic of federated learning. His research works have won multiple awards from conferences and journals.
\end{IEEEbiographynophoto}
\vspace{-0.5em}

\begin{IEEEbiographynophoto}{Yang Yu}
 received the BSc and PhD degree in computer science from Nanjing University, China, in 2004 and 2011, respectively. He joined the Department of Computer Science \& Technology at Nanjing University as an Assistant Researcher in 2011, and is currently Professor of the School of Artificial Intelligence. He has co-authored the book Evolutionary Learning: Advances in Theories and Algorithms, and published more than 40 papers in top-tier international journals and and conference proceedings. He has been recognized as a AI's 10 to Watch by IEEE Intelligent Systems (2018), PAKDD Early Career Award (2018), CCF-IEEE CS Early Career Young Scientist (2020), and was invited to give an IJCAI'18 Early Career Spotlight. He co-founded the Asian Workshop on Reinforcement Learning.
\end{IEEEbiographynophoto}
\vspace{-0.5em}

\begin{IEEEbiographynophoto}{Zhi-Hua Zhou} (S’00-M’01-SM’06-F’13) received the BSc, MSc and PhD degrees in computer science from Nanjing University, China, in 1996, 1998 and 2000, respectively, all with the highest honors. He joined the Department of Computer Science \& Technology at Nanjing University as an Assistant Professor in 2001, and is currently Professor, Head of the Department of Computer Science and Technology, and Dean of the School of Artificial Intelligence; he is also the Founding Director of the LAMDA group. His research interests are mainly in artificial intelligence, machine learning and data mining. He has authored the books Ensemble Methods: Foundations and Algorithms, Evolutionary Learning: Advances in Theories and Algorithms, Machine Learning (in Chinese), and published more than 150 papers in top-tier international journals or conference proceedings. He has received various awards/honors including the National Natural Science Award of China, the IEEE Computer Society Edward J. McCluskey Technical Achievement Award, the PAKDD Distinguished Contribution Award, the IEEE ICDM Outstanding Service Award, the Microsoft Professorship Award, etc. He also holds 24 patents. He is the Editor-in-Chief of the Frontiers of Computer Science, Associate Editor-in-Chief of the Science China Information Sciences, Action or Associate Editor of the Machine Learning, IEEE Transactions on Pattern Analysis and Machine Intelligence , ACM Transactions on Knowledge Discovery from Data, etc. He served as Associate Editor-in-Chief for Chinese Science Bulletin (2008- 2014), Associate Editor for IEEE Transactions on Knowledge and Data Engineering (2008-2012), IEEE Transactions on Neural Networks and Learning Systems (2014-2017), ACM Transactions on Intelligent Systems and Technology (2009-2017), Neural Networks (2014-2016), etc. He founded ACML (Asian Conference on Machine Learning), served as Advisory Committee member for IJCAI (2015-2016), Steering Committee member for ICDM, PAKDD and PRICAI, and Chair of various conferences such as Program co-chair of AAAI 2019, General co-chair of ICDM 2016, and Area chair of NeurIPS, ICML, AAAI, IJCAI, KDD, etc. He was the Chair of the IEEE CIS Data Mining Technical Committee (2015-2016), the Chair of the CCF-AI (2012-2019), and the Chair of the CAAI Machine Learning Technical Committee (2006-2015). He is a foreign member of the Academy of Europe, and a Fellow of the ACM, AAAI, AAAS, IEEE, IAPR, IET/IEE, CCF, and CAAI.
\end{IEEEbiographynophoto}





\begin{thebibliography}{10}

\bibitem{agarwal2019general}
Aman Agarwal, Kenta Takatsu, Ivan Zaitsev, and Thorsten Joachims.
\newblock A general framework for counterfactual learning-to-rank.
\newblock In {\em Proceedings of the 42nd International ACM SIGIR Conference on
  Research and Development in Information Retrieval}, pages 5--14, 2019.

\bibitem{ai2018learning}
Qingyao Ai, Keping Bi, Jiafeng Guo, and W~Bruce Croft.
\newblock Learning a deep listwise context model for ranking refinement.
\newblock In {\em The 41st International ACM SIGIR Conference on Research \&
  Development in Information Retrieval}, pages 135--144. ACM, 2018.

\bibitem{ai2019learning}
Qingyao Ai, Xuanhui Wang, Sebastian Bruch, Nadav Golbandi, Michael Bendersky,
  and Marc Najork.
\newblock Learning groupwise multivariate scoring functions using deep neural
  networks.
\newblock In {\em Proceedings of the 2019 ACM SIGIR International Conference on
  Theory of Information Retrieval}, pages 85--92. ACM, 2019.

\bibitem{beel2013comparative}
Joeran Beel, Marcel Genzmehr, Stefan Langer, Andreas N{\"u}rnberger, and Bela
  Gipp.
\newblock A comparative analysis of offline and online evaluations and
  discussion of research paper recommender system evaluation.
\newblock In {\em Proceedings of the international workshop on reproducibility
  and replication in recommender systems evaluation}, pages 7--14, 2013.

\bibitem{seq2slate}
Irwan Bello, Sayali Kulkarni, Sagar Jain, Craig Boutilier, Ed~Huai{-}hsin Chi,
  Elad Eban, Xiyang Luo, Alan Mackey, and Ofer Meshi.
\newblock Seq2slate: Re-ranking and slate optimization with rnns.
\newblock {\em CoRR}, abs/1810.02019, 2018.

\bibitem{burgesRL06pair}
Christopher J.~C. Burges, Robert Ragno, and Quoc~Viet Le.
\newblock Learning to rank with nonsmooth cost functions.
\newblock In {\em Advances in Neural Information Processing Systems 19,
  Proceedings of the Twentieth Annual Conference on Neural Information
  Processing Systems, Vancouver, British Columbia, Canada, December 4-7, 2006},
  pages 193--200, 2006.

\bibitem{Burges:ranknet}
Christopher J.~C. Burges, Tal Shaked, Erin Renshaw, Ari Lazier, Matt Deeds,
  Nicole Hamilton, and Gregory~N. Hullender.
\newblock Learning to rank using gradient descent.
\newblock In {\em Machine Learning, Proceedings of the Twenty-Second
  International Conference {(ICML} 2005)}, pages 89--96, 2005.

\bibitem{burges2010ranknet}
Christopher~JC Burges.
\newblock From ranknet to lambdarank to lambdamart: An overview.
\newblock {\em Learning}, 11(23-581):81, 2010.

\bibitem{cao07list}
Zhe Cao, Tao Qin, Tie{-}Yan Liu, Ming{-}Feng Tsai, and Hang Li.
\newblock Learning to rank: from pairwise approach to listwise approach.
\newblock In {\em Machine Learning, Proceedings of the Twenty-Fourth
  International Conference {(ICML} 2007), Corvallis, Oregon, USA, June 20-24,
  2007}, pages 129--136, 2007.

\bibitem{topkoff}
Minmin Chen, Alex Beutel, Paul Covington, Sagar Jain, Francois Belletti, and
  Ed~H. Chi.
\newblock Top-k off-policy correction for a {REINFORCE} recommender system.
\newblock In {\em Proceedings of the Twelfth {ACM} International Conference on
  Web Search and Data Mining, {WSDM} 2019, Melbourne 11-15, 2019}, pages
  456--464, 2019.

\bibitem{cossock08point}
David Cossock and Tong Zhang.
\newblock Statistical analysis of bayes optimal subset ranking.
\newblock {\em {IEEE} Trans. Information Theory}, 54(11):5140--5154, 2008.

\bibitem{dacrema2019we}
Maurizio~Ferrari Dacrema, Paolo Cremonesi, and Dietmar Jannach.
\newblock Are we really making much progress? a worrying analysis of recent
  neural recommendation approaches.
\newblock In {\em Proceedings of the 13th ACM Conference on Recommender
  Systems}, pages 101--109, 2019.

\bibitem{duan2017one}
Yan Duan, Marcin Andrychowicz, Bradly Stadie, OpenAI~Jonathan Ho, Jonas
  Schneider, Ilya Sutskever, Pieter Abbeel, and Wojciech Zaremba.
\newblock One-shot imitation learning.
\newblock In {\em Advances in neural information processing systems}, pages
  1087--1098, 2017.

\bibitem{finn2016guided}
Chelsea Finn, Sergey Levine, and Pieter Abbeel.
\newblock Guided cost learning: Deep inverse optimal control via policy
  optimization.
\newblock In {\em International Conference on Machine Learning}, pages 49--58,
  2016.

\bibitem{gong19exact}
Yu~Gong, Yu~Zhu, Lu~Duan, Qingwen Liu, Ziyu Guan, Fei Sun, Wenwu Ou, and
  Kenny~Q. Zhu.
\newblock Exact-k recommendation via maximal clique optimization.
\newblock In {\em Proceedings of the 25th {ACM} {SIGKDD} International
  Conference on Knowledge Discovery {\&} Data Mining, {KDD} 2019, Anchorage,
  AK, USA, August 4-8, 2019}, pages 617--626, 2019.

\bibitem{goodfellow15gan}
Ian~J. Goodfellow, Jean Pouget{-}Abadie, Mehdi Mirza, Bing Xu, David
  Warde{-}Farley, Sherjil Ozair, Aaron~C. Courville, and Yoshua Bengio.
\newblock Generative adversarial nets.
\newblock In {\em Annual Conference on Neural Information Processing Systems
  2014, December 8-13 2014, Montreal}, pages 2672--2680, 2014.

\bibitem{ho16gail}
Jonathan Ho and Stefano Ermon.
\newblock Generative adversarial imitation learning.
\newblock In {\em Annual Conference on Neural Information Processing Systems
  2016, December 5-10, 2016, Barcelona, Spain}, pages 4565--4573, 2016.

\bibitem{ie2019recsim}
Eugene Ie, Chih-wei Hsu, Martin Mladenov, Vihan Jain, Sanmit Narvekar, Jing
  Wang, Rui Wu, and Craig Boutilier.
\newblock Recsim: A configurable simulation platform for recommender systems.
\newblock {\em arXiv preprint arXiv:1909.04847}, 2019.

\bibitem{slateQ}
Eugene Ie, Vihan Jain, Jing Wang, Sanmit Narvekar, Ritesh Agarwal, Rui Wu,
  Heng{-}Tze Cheng, Tushar Chandra, and Craig Boutilier.
\newblock Slateq: {A} tractable decomposition for reinforcement learning with
  recommendation sets.
\newblock In {\em Proceedings of the Twenty-Eighth International Joint
  Conference on Artificial Intelligence, {IJCAI} 2019, August 10-16, 2019},
  pages 2592--2599, 2019.

\bibitem{jiang19cave}
Ray Jiang, Sven Gowal, Yuqiu Qian, Timothy~A. Mann, and Danilo~J. Rezende.
\newblock Beyond greedy ranking: Slate optimization via list-cvae.
\newblock In {\em 7th International Conference on Learning Representations,
  {ICLR} 2019, New Orleans, LA, USA, May 6-9, 2019}, 2019.

\bibitem{joachims02pair}
Thorsten Joachims.
\newblock Optimizing search engines using clickthrough data.
\newblock In {\em Proceedings of the Eighth {ACM} {SIGKDD} International
  Conference on Knowledge Discovery and Data Mining, July 23-26, 2002,
  Edmonton, Alberta, Canada}, pages 133--142, 2002.

\bibitem{joachims2017unbiased}
Thorsten Joachims, Adith Swaminathan, and Tobias Schnabel.
\newblock Unbiased learning-to-rank with biased feedback.
\newblock In {\em Proceedings of the Tenth ACM International Conference on Web
  Search and Data Mining}, pages 781--789, 2017.

\bibitem{kool2019buy}
Wouter Kool, Herke van Hoof, and Max Welling.
\newblock Buy 4 reinforce samples, get a baseline for free!
\newblock 2019.

\bibitem{li07point}
Ping Li, Christopher J.~C. Burges, and Qiang Wu.
\newblock Mcrank: Learning to rank using multiple classification and gradient
  boosting.
\newblock In {\em Proceedings of the Twenty-First Annual Conference on Neural
  Information Processing Systems, Vancouver, British Columbia, Canada, December
  3-6, 2007}, pages 897--904, 2007.

\bibitem{ma2018entire}
Xiao Ma, Liqin Zhao, Guan Huang, Zhi Wang, Zelin Hu, Xiaoqiang Zhu, and Kun
  Gai.
\newblock Entire space multi-task model: An effective approach for estimating
  post-click conversion rate.
\newblock In {\em The 41st International ACM SIGIR Conference on Research \&
  Development in Information Retrieval}, pages 1137--1140, 2018.

\bibitem{mcnee2006being}
Sean~M McNee, John Riedl, and Joseph~A Konstan.
\newblock Being accurate is not enough: how accuracy metrics have hurt
  recommender systems.
\newblock In {\em CHI'06 extended abstracts on Human factors in computing
  systems}, pages 1097--1101, 2006.

\bibitem{pei2019personalized}
Changhua Pei, Yi~Zhang, Yongfeng Zhang, Fei Sun, Xiao Lin, Hanxiao Sun, Jian
  Wu, Peng Jiang, Junfeng Ge, Wenwu Ou, et~al.
\newblock Personalized re-ranking for recommendation.
\newblock In {\em Proceedings of the 13th ACM Conference on Recommender
  Systems}, pages 3--11, 2019.

\bibitem{rohde2018recogym}
David Rohde, Stephen Bonner, Travis Dunlop, Flavian Vasile, and Alexandros
  Karatzoglou.
\newblock Recogym: A reinforcement learning environment for the problem of
  product recommendation in online advertising.
\newblock {\em arXiv preprint arXiv:1808.00720}, 2018.

\bibitem{ross2010efficient}
St{\'e}phane Ross and Drew Bagnell.
\newblock Efficient reductions for imitation learning.
\newblock In {\em Proceedings of the thirteenth international conference on
  artificial intelligence and statistics}, pages 661--668, 2010.

\bibitem{ross2011reduction}
St{\'e}phane Ross, Geoffrey Gordon, and Drew Bagnell.
\newblock A reduction of imitation learning and structured prediction to
  no-regret online learning.
\newblock In {\em Proceedings of the fourteenth international conference on
  artificial intelligence and statistics}, pages 627--635, 2011.

\bibitem{rossetti2016contrasting}
Marco Rossetti, Fabio Stella, and Markus Zanker.
\newblock Contrasting offline and online results when evaluating recommendation
  algorithms.
\newblock In {\em Proceedings of the 10th ACM conference on recommender
  systems}, pages 31--34, 2016.

\bibitem{schulman2015trust}
John Schulman, Sergey Levine, Pieter Abbeel, Michael~I. Jordan, and Philipp
  Moritz.
\newblock Trust region policy optimization.
\newblock In {\em Proceedings of the 32nd International Conference on Machine
  Learning, {ICML} 2015}, pages 1889--1897, 2015.

\bibitem{schulman2017proximal}
John Schulman, Filip Wolski, Prafulla Dhariwal, Alec Radford, and Oleg Klimov.
\newblock Proximal policy optimization algorithms.
\newblock {\em CoRR}, abs/1707.06347, 2017.

\bibitem{shi2019virtual}
Jing-Cheng Shi, Yang Yu, Qing Da, Shi-Yong Chen, and An-Xiang Zeng.
\newblock Virtual-taobao: Virtualizing real-world online retail environment for
  reinforcement learning.
\newblock In {\em Proceedings of the AAAI Conference on Artificial
  Intelligence}, volume~33, pages 4902--4909, 2019.

\bibitem{sutton1998introduction}
RS~Sutton and AG~Barto.
\newblock {\em Introduction to reinforcement learning}, volume 135.
\newblock MIT press Cambridge, 1998.

\bibitem{taylor2008softrank}
Michael Taylor, John Guiver, Stephen Robertson, and Tom Minka.
\newblock Softrank: optimizing non-smooth rank metrics.
\newblock In {\em Proceedings of the 2008 International Conference on Web
  Search and Data Mining}, pages 77--86, 2008.

\bibitem{vinyals2015pointer}
Oriol Vinyals, Meire Fortunato, and Navdeep Jaitly.
\newblock Pointer networks.
\newblock In {\em Advances in Neural Information Processing Systems}, pages
  2692--2700, 2015.

\bibitem{wagstaff2012machine}
Kiri Wagstaff.
\newblock Machine learning that matters.
\newblock {\em arXiv preprint arXiv:1206.4656}, 2012.

\bibitem{wang19baidu}
Fan Wang, Xiaomin Fang, Lihang Liu, Yaxue Chen, Jiucheng Tao, Zhiming Peng,
  Cihang Jin, and Hao Tian.
\newblock Sequential evaluation and generation framework for combinatorial
  recommender system.
\newblock {\em CoRR}, abs/1902.00245, 2019.

\bibitem{wu2018turning}
Liang Wu, Diane Hu, Liangjie Hong, and Huan Liu.
\newblock Turning clicks into purchases: Revenue optimization for product
  search in e-commerce.
\newblock In {\em {SIGIR}}, pages 365--374. {ACM}, 2018.

\bibitem{xia08list}
Fen Xia, Tie{-}Yan Liu, Jue Wang, Wensheng Zhang, and Hang Li.
\newblock Listwise approach to learning to rank: theory and algorithm.
\newblock In {\em Machine Learning, Proceedings of the Twenty-Fifth
  International Conference {(ICML} 2008), Helsinki, Finland, June 5-9, 2008},
  pages 1192--1199, 2008.

\bibitem{ZhangMLZ0MXT19}
Junqi Zhang, Jiaxin Mao, Yiqun Liu, Ruizhe Zhang, Min Zhang, Shaoping Ma, Jun
  Xu, and Qi~Tian.
\newblock Context-aware ranking by constructing a virtual environment for
  reinforcement learning.
\newblock In {\em {CIKM}}, pages 1603--1612. {ACM}, 2019.

\bibitem{zhao2018deep}
Xiangyu Zhao, Long Xia, Liang Zhang, Zhuoye Ding, Dawei Yin, and Jiliang Tang.
\newblock Deep reinforcement learning for page-wise recommendations.
\newblock In {\em Proceedings of the 12th ACM Conference on Recommender
  Systems}, pages 95--103. ACM, 2018.

\bibitem{zhou2018deep}
Guorui Zhou, Xiaoqiang Zhu, Chenru Song, Ying Fan, Han Zhu, Xiao Ma, Yanghui
  Yan, Junqi Jin, Han Li, and Kun Gai.
\newblock Deep interest network for click-through rate prediction.
\newblock In {\em Proceedings of the 24th ACM SIGKDD International Conference
  on Knowledge Discovery \& Data Mining}, pages 1059--1068. ACM, 2018.

\bibitem{Zhuang18}
Tao Zhuang, Wenwu Ou, and Zhirong Wang.
\newblock Globally optimized mutual influence aware ranking in e-commerce
  search.
\newblock In {\em Proceedings of the Twenty-Seventh International Joint
  Conference on Artificial Intelligence, {IJCAI} 2018, July 13-19, 2018,
  Stockholm, Sweden}, pages 3725--3731, 2018.

\end{thebibliography}
\end{document}